\documentclass[letterpaper, 10 pt, conference]{ieeeconf}
\IEEEoverridecommandlockouts                              
\makeatletter
\let\NAT@parse\undefined
\makeatother

\usepackage{comment}

\IEEEoverridecommandlockouts                              

\usepackage{epsfig} 
\usepackage{mathptmx} 
\usepackage{times} 
\usepackage{amsmath} 
\usepackage{amssymb}  
\usepackage{multirow}
\usepackage{array}
\usepackage{tabularx}
\usepackage{cite}
\usepackage{footnote}
\makesavenoteenv{tabular} 
\usepackage{gensymb}
\usepackage{bm} 
\usepackage{mathtools}
\usepackage{algorithm,algorithmic}
\usepackage{graphicx}
\usepackage{booktabs}
\floatname{algorithm}{Method}
\usepackage{color,soul}
\usepackage{comment}
\usepackage{relsize}
\usepackage{tikz}

\makeatletter
\let\NAT@parse\undefined
\makeatother
\usepackage{booktabs}
\usepackage[utf8]{inputenc}
\usepackage{tabularx,ragged2e,booktabs}
\usepackage{amsmath}
\usepackage{bbm}

\usepackage{epsfig}
\usepackage{epstopdf}
\usepackage{bigstrut}
\usepackage{graphicx}
\usepackage{epsfig}
\usepackage{amsmath}
\usepackage{amssymb} 
\usepackage{longtable}
\usepackage{rotating}
\usepackage{multirow}
\usepackage{array}
\usepackage{graphicx}
\usepackage{mathrsfs}
\usepackage{verbatim}
\usepackage{color}

\usepackage{url}
\usepackage{graphicx}

\usepackage{verbatim}
\usepackage{soul}
\usepackage{algorithm}
\usepackage{algorithm}
\usepackage{epstopdf}
\usepackage{url}
\usepackage{pgfplots,tikz-3dplot}
\usepackage{amsmath} 
\usepackage{amssymb}  
\usepackage{bbm}
\usepackage{color}
\usepackage{graphics}
\usepackage{graphicx}
\usepackage{epsfig}
\usepackage{epstopdf}
\usepackage{amsfonts}
\usepackage{mathtools}
\usepackage{bm}
\usepackage{fancyhdr}
\usepackage{framed}
\usepackage{float}
\usepackage{cite}
\usepackage{verbatim}

\usepackage{relsize}
\usepackage{mathrsfs}
\usepackage{todonotes}
\usepackage{pgfplots}
\usepackage{grffile}
\usepackage{hyperref}

\usepackage{tikz}
\usetikzlibrary{calc,positioning,shapes,shadows,arrows,fit}
\newcommand{\bunderline}[1]{\underline{#1\mkern-4mu}\mkern4mu }
\newcommand{\overbar}[1]{\mkern 1.0mu\overline{\mkern-1.0mu#1\mkern-1.0mu}\mkern 1.0mu}

\usepackage[caption=false, font=normalsize, labelfont=sf, textfont=sf]{subfig}

\usetikzlibrary{arrows}

\newcommand{\diag}{\mathop{\mathrm{diag}}}
\usepackage{bigints}
\usepackage{todonotes}

\newtheorem{definition}{\bf{Definition}}
\newtheorem{theorem}{\bf{Theorem}}

\newtheorem{lem}{\bf{Lemma}}
\newcommand{\bs}{\boldsymbol}

\newcolumntype{C}[1]{>{\centering\let\newline\\\arraybackslash}m{#1}}

\title{\LARGE \bf
Safe Force/Position Tracking Control via Control Barrier Functions for Floating Base Mobile Manipulator Systems}
\author{Maryam Sharifi, Shahab Heshmati-Alamdari 
\thanks{ Maryam Sharifi is with ABB Corporate Research, Västerås, Sweden, Email: {\tt\small  \{maryam.sharifi@se.abb.com\}}. Shahab Heshmati-alamdari is with the Section of Automation \& Control, Department of Electronic Systems, Aalborg University, Denmark, Email: {\tt\small \{shhe@se.aau.dk\}}.
}
}
\begin{document}

\maketitle

\begin{abstract}
This paper introduces a safe force/position tracking control strategy designed for Free-Floating Mobile Manipulator Systems (MMSs) engaging in compliant contact with planar surfaces. The strategy uniquely integrates the Control Barrier Function (CBF) to manage operational limitations and safety concerns. It effectively addresses safety-critical aspects in the kinematic as well as dynamic level, such as manipulator joint limits, system velocity constraints, and inherent system dynamic uncertainties. The proposed strategy remains robust to the uncertainties of the MMS dynamic model, external disturbances, or variations in the contact stiffness model. The proposed control method has low computational demand ensures easy implementation on onboard computing systems, endorsing real-time operations. Simulation results verify the strategy's efficacy, reflecting enhanced system performance and safety. 
\end{abstract}
\section{Introduction}\label{sec:introduction}
In recent years, a growing number of robotic tasks, beyond just mobility, necessitate autonomous robots cooperating as a team \cite{nikou2017decentralized} and being equipped with interaction capabilities. This development positions them as suitable substitutes for humans in various applications\cite{sereinig2020review}. Mobile Manipulator Systems (MMSs) \cite{heshmati2023control} can be classified from various angles, including the domain of application or their specific type \cite{ baizid2017behavioral}.  A particularly complex category within MMSs is identified as Floating Base Mobile Manipulator Systems (FBMMSs) \cite{karras2022image}. These systems feature a vehicle that is not anchored, effectively putting the base of the manipulator in a "flying mode". Such a configuration grants the system the liberty to navigate through all axes of movement, independent of the actuation of these movements.
\begin{figure}[h]
	\centering
	\setlength{\fboxsep}{0pt}%
	\setlength{\fboxrule}{2pt}%
	\includegraphics[width=0.34\textwidth]{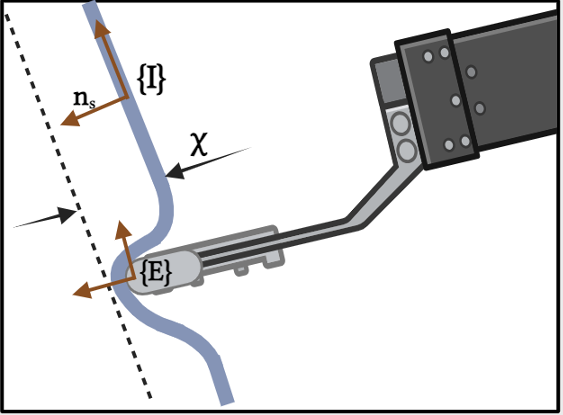}
	\caption{A graphical illustration of the MMS end-effector in compliant contact with a planar surface.}
	\label{fig:comp_envir}      
\end{figure}
When using MMSs, certain unique characteristics and constraints arise, some of which are determined by the system's operational domain. These nuances often require a specific array of onboard sensors and careful consideration of their specifications during control design. Additionally, during the control design for these systems, it's essential to meet safety constraints like joint limits and prevent system singularities, while also considering performance metrics such as system manipulability \cite{heshmati2019distributed}. For many tasks involving MMs, especially those requiring precision like object retrieval, human-robot interactions, or specific tasks like welding, it's essential for the system to interact precisely with its environment. Such tasks often call for adherence to preset performance standards, whether it's maintaining contact, as in welding, or preventing force overshoots, vital in delicate interactions, or sampling from sensitive environments \cite{heshmati2017robust}. In applications like welding or grinding, these systems are tasked with interacting with restricted surfaces. Consequently, tracking the intended trajectory and adhering to the desired force constraints have emerged as complex issues in the research community \cite{rani2018efficient}.

While the literature on trajectory tracking control is extensive, studies on position/force tracking control of mobile manipulators are notably sparse. Control over these manipulators' desired position/force is often achieved through adaptive or robust control strategies \cite{dong2002trajectory}. Various robust and/or adaptive methodologies have been investigated to counter uncertainties in motion and force control, while considering specific challenges such as actuator dynamics, or non-holonomic base features. However, despite these advancements, a holistic approach that simultaneously prioritizes safety in both system and task execution performance and specification remains relatively unexplored \cite{sereinig2020review}.

In this work, we incorporate safety constraints through control barrier functions (CBFs), leading to control laws derived by implementing quadratic programming (QP) that ensures the system's safety. This approach has been applied across a variety of applications, including robotics \cite{wang2017safety}, multi-agent systems \cite{ sharifi2022higher}, and automated vehicles \cite{alan2023control}, among others. It has already been noted that safety constraints concern not only the kinematic behavior of the FBMMS but also heavily depend on the system's dynamic model. We address all these constraints using zeroing control barrier functions (ZCBFs), which are well-defined on the boundary and exterior of the safe sets, unlike the reciprocal CBFs \cite{kolathaya2023energy}. The solution's consistency with respect to both the kinematics and dynamics of the system is ensured, and the modelling uncertainties of the system dynamics are handled by a robust QP formulation.

\section{Problem Formulation}\label{Section:Problem}
Consider a FBMMS with $n$ degrees of freedom that's in compliant contact with a flat surface. Define the state variables of the FBMMS as $\bs{q}=[\bs{q}^\top_a,~\bs{q}^\top_m]^\top\in \mathbb{R}^n$. Here, $\bs{q}_a=[\bs{\eta}_1^\top,\bs{\eta_2}^\top]^\top\in \mathbb{R}^6$ encompasses the position vector $\bs{\eta}_{1}=[x_v,y_v,z_v]^\top$ and the orientation $\bs{\eta}_{2}=[\phi,\theta,\psi]^\top$ of the vehicle, depicted using the Euler-angles in relation to an inertial frame $\{I\}$. Furthermore, $\bs{q}_m\in \mathbb{R}^{n-6}$ represents the angle vector for the manipulator's joints. Introducing the frame $\{E\}$ positioned at the FBMMS's end-effector, it's defined by a position vector $\bs{x}_e=[x_e,y_e,z_e]^\top\in \mathbb{R}^3$ and an orientation matrix $\bs{R}_e=[\bs{n}_e,\bs{o}_e,\bs{\alpha}_e]$ in relation to the inertial frame $\{I\}$. Additionally, let $\bs{\omega}_e$ denote the end-effector's rotational velocity, for which $\bs{S}(\bs{\omega}_e)=\dot{\bs{R}}_e\bs{R}_e^\top$. The skew-symmetric matrix for the vector $\bs{d}=[d_x~d_y~d_z]^\top$ can be expressed as:
\begin{align*}
\bs{S}(\bs{d})=\begin{bmatrix}
0 & -d_z & d_y\\
d_z &  0 & -d_x\\
-d_y & d_x & 0 
\end{bmatrix}
\end{align*}. Consider $\dot{\bs{x}}$ as an element of $\mathbb{R}^6$ representing the velocity of the end-effector frame.
Drawing from the findings in \cite{antonelli}, one can express:
\begin{equation}
\dot{\bs{x}} = \bs{J}(\bs{q})\boldsymbol{\zeta}\label{eq222}
\end{equation}
Here, $\boldsymbol{\zeta}=[\bs{v}^\top,\bs{\zeta}_{m,i}^\top]^\top \in \mathbb{R}^{n}$  encompasses the velocity components: the vehicle's body velocities $\bs{v}$ and the joint velocities of the manipulator, represented as $\bs{\zeta}_{m,i}$ for $i$ ranging from 1 to $n-6$. Additionally, $\bs{J}(\bs{q})$ of dimension $\mathbb{R}^{6\times n}$ is identified as the geometric Jacobian matrix, as detailed in \cite{antonelli}.

Assuming initial contact of the MMS with a planar surface, knowledge of the normal and tangential vectors in relation to the inertial frame $\{I\}$ is presumed. For clarity in representation, let the inertial frame $\{I\}$ be anchored at a specific point on the surface. The frame's $x$-axis, oriented normal to the surface, points inward as illustrated in Fig.\ref{fig:comp_envir}.

Now, let us denote the unit vector normal to the contact surface and the generalized normal vector as $\bs{n}_s=[1~0~0]^\top\in \mathbb{R}^3$  and $\bs{n}=[\bs{n}_s^\top~\bs{0}_3^\top]^\top\in \mathbb{R}^6$ respectively.

We also assume that the end-effector is rigid, thus the contact compliance arises from the planar surface\footnote{In case of FBMMS with soft tip, the compliance may arise either from the tip side or the surface or both. Thus, the deformation can also be derived without affecting the subsequent analysis.}. Hence, the deformation $\chi$ is given as a function of $\bs{x}_e$ as follows:
\begin{align}
\chi=\bs{n}_s^\top\bs{x}_e=x_e\label{deformation}
\end{align}
and its derivative is calculated by:
\begin{align}
\dot{\chi}=\bs{n}_s^\top\dot{\bs{x}}_e= \dot{x}_e\label{dot_defor}
\end{align}
During an intervention task, the FBMMS exerts an interaction wrench $\bs{\lambda}\in\mathbb{R}^6$ at the contact, which can be measured by a force/torque sensor attached to its end-effector. This interaction wrench can be decomposed into: (i) $\bs{n}\bs{n}^\top\bs{\lambda}$ that is normal to the surface and (ii) $(\bs{I}_{6\times 6}-\bs{n}\bs{n}^\top)\bs{\lambda}$ involving tangential forces and torques, owing to tangential deformations and friction terms. In this work, we assume that the normal force magnitude $f=\bs{n}^\top\bs{\lambda}$ is a positive and continuously differentiable nonlinear function of the material deformation $\chi$:
\begin{align}
f=\Phi(\chi), \quad \forall \chi\geq 0.\label{magn_force}
\end{align}
The aforementioned general formulation includes several force deformation models such as the Hertz model \cite{johnson1987contact} ($\Phi(\chi)= k\chi^{\frac{3}{2}},k>0$) or the quadratic model $\Phi(\chi)=k\chi^2,k>0$  \cite{arimoto2000dynamics}. 
The time derivative of the normal force magnitude in view of \eqref{dot_defor} is then given by:
\begin{align}
\dot{f}= \partial \Phi(\chi)\dot{x}_e \label{dot_f}
\end{align}
where $\partial\Phi(\chi)=\frac{d\Phi}{d\chi}$ is strictly positive for all $\chi\geq\bunderline{\chi}^*>0$, where $\bunderline{\chi}^*$ is any strictly real positive number. Thus, there is an unknown strictly positive constant $\partial \Phi^*$ such that:
\begin{align}
\partial\Phi(\chi)\geq\partial\Phi^*>0, \forall \chi\geq\bunderline{\chi}^*.\label{bound_part_force}
\end{align}

Without loss  of generality, the dynamics of the FBMMS in complaint contact with the environment is formulated as \cite{antonelli}:
\begin{gather}
\bs{M}(\bs{q})\dot{\bs{\zeta}}\!+\!\bs{C}({\bs{q}},\bs{\zeta}){\bs{\zeta}}\!+\!\bs{D}({\bs{q}},\bs{\zeta}){\bs{\zeta}}\!+\!\bs{g}(
\bs{q})\!+\!{\bs{J}}^\top(\bs{q})\boldsymbol{\lambda}\!+\!\boldsymbol{\delta}(\bs{q},\bs{\zeta},t)\!=\!\bs{\tau}\label{eq4}
\end{gather}
where $\boldsymbol{\delta}(\bs{q},\bs{\zeta},t)$ encapsulates bounded unmodeled terms and external disturbances (sea waves and currents). Moreover, $\bs{\tau} \in \mathbb{R}^n$ denotes the control input at the joint/thruster level, $\bs{{M}}(\bs{q})$ is the positive definite inertial matrix, $\bs{{C}}({\bs{q}},\bs{\zeta})$ represents coriolis and centrifugal terms, $\bs{{D}}({\bs{q}},\bs{\zeta})$ models dissipative effects, $\bs{{g}}(\bs{q})$ encapsulates the gravity and buoyancy effects and $\bs{J}^\top(\bs{q})\bs{\lambda}$ represents the effect of the external forces/torques applied at the end-effector owing to the contact.


\section{Control Methodology}\label{control_method}
When addressing force/position control challenges, it's crucial that the robot's end-effector adheres to a predefined force path normal to a contact surface, maintains the target position trajectory on that surface, and achieves the desired orientation in relation to the surface. The errors for force, position, orientation, and the cumulative error vector are described as follows:
\begin{subequations}\label{eq8}
	\begin{align}
	&{e}_f= f- f^d\label{err_1}~\in \mathbb{R},\\
	&\bs{e}_p\triangleq\begin{bmatrix}
	e_{y}\\e_{z}
	\end{bmatrix}=\begin{bmatrix}
    y_e - y_e^d\\
     z_e- z_e^d
	\end{bmatrix}~\in \mathbb{R}^2,\label{err_2}\\
	&\bs{e}_o \triangleq\begin{bmatrix}
	e_{o_1}\\e_{o_2}\\e_{o_3}
	\end{bmatrix}=\frac{1}{2}\Big(\bs{n}_e\times\bs{n}^d+\bs{o}_e\times\bs{o}^d+\bs{\alpha}_e\times\bs{\alpha}^d\Big)\label{err_4}~\in\mathbb{R}^3,\\
	&~\bs{e} \triangleq [e_f,e_y,e_z,e_{o_1},e_{o_2},e_{o_3}]^\top\in\mathbb{R}^6.\label{err_3}
	\end{align}
\end{subequations}
For the calculation of the orientation error $\boldsymbol{e}_o$, the cross product operation between $\bs{R}_e$ and $\bs{R}^d$ is utilized to address the singularities often present when using Euler angles for rotation representations, as suggested in the literature \cite{caccavale1998resolved, sciavicco2012modelling}. 

In view of \eqref{dot_f}, time differentiation of errors \eqref{err_1}-\eqref{err_4} leads to their respective rates as:
\begin{subequations}\label{err_dot}
	\begin{align}
	&\dot{{e}}_f= \partial f(\chi) \dot{{x}}_e- \dot{f}^d \label{err_dot_1},\\
	&\dot{\bs{e}}_p=\begin{bmatrix}
	\dot{e}_{y}\\\dot{e}_{z}
	\end{bmatrix}=\begin{bmatrix}
	\dot{y}_{e}- \dot{y}^d_{e}\\
	\dot{z}_{e}- \dot{z}^d_{e}
	\end{bmatrix}\label{err_dot_2},\\
	&\dot{\bs{e}}_o =\begin{bmatrix}
	\dot{e}_{o_1}\\\dot{e}_{o_2}\\\dot{e}_{o_3}
	\end{bmatrix}=\bs{L}\bs{\omega}_e-\bs{L}\bs{\omega}^d, \label{err_dot_3}
	\end{align}
\end{subequations}
where $\bs{L}$ is defined as:
\begin{align}
\bs{L}=\frac{1}{2}\Big[   \bs{S}(\bs{n}_e)\bs{S}(\bs{n}^d)     +\bs{S}(\bs{o}_e)\bs{S}(\bs{o}^d) + \bs{S}(\bs{\alpha}_e)   \bs{S}(\bs{\alpha}^d)        \Big]  \label{L_defin}
\end{align}
which is full rank when the relative orientation between the frames $\bs{R}_e$ and $\bs{R}^d$ is confined less than $90^\circ$ for an angle-axis local parametrization and hence is not restrictive for practical cases \cite{sciavicco2012modelling}. 

\subsection{Control Design}
In this section, zeroing control barrier functions (ZCBF) which guarantees the forward invariance of the corresponding safe sets are used to meet the control objectives. Consider the force, position and orientation errors evolve strictly within predefined safe regions that are  mathematically expressed as: 
\begin{align}
-\bunderline{M}_{i}<e_{i}(t)<\overline{M}_{i},~i\in\{f,y,z,o_1,o_2,o_3\},~\forall t\geq 0\label{eq:11}
\end{align}

The constants $ \bunderline{M}_i,~ \overline{M}_i$, are selected such that \eqref{eq:11} is satisfied at $t=0$ (i.e.,  $-\bunderline{M}_i< e_i(0)<\overline{M}_i$). 
The bounds of the errors $e_i(t)$, and actually the maximum overshoot is defined to be less than $\bunderline{M}_i)$ or  $\overline{M}_i$. Thus, the appropriate selection  of the design constants $\bunderline{M}_i$, $\overline{M}_{i}$, $i\in\{f,y,z,o_1,o_2,o_3\}$ encapsulates performance characteristics for the corresponding tracking errors  $e_{i}$, $i\in\{f,y,z,o_1,o_2,o_3\}$.

The satisfaction of the performance criteria for force error, as outlined in \eqref{eq:11}, paves the way for ensuring two critical conditions. Firstly, contact with the surface remains unbroken, meaning $f(t) \geq \bunderline{f}^*>0,~\forall t\geq 0$, where $\bunderline{f}^*$ denotes a positive constant. Secondly, there's prevention against incurring inordinately high interaction forces, illustrated as $f(t)\leq \overline{f}^*>0,~\forall t\geq0$, with $\overline{f}^*$ being another positive constant that is strictly greater than $\bunderline{f}^*$. Building on this foundation, criteria for $\bunderline{M}_f$ and $\overline{M}_f$ are chosen to meet: $\inf_{t\geq0}\{ -\bunderline{M}_f+f^d(t) \}>\bunderline{f}^*$ and $\sup_{t\geq0}\{ \overline{M}_f+f^d(t) \}<\overline{f}^*$. These criteria confirm that for any $t\geq 0$, $f(t)$ resides within the bounds given by $\inf_{t\geq0}\{ -\bunderline{M}_f+f^d(t) \}$ and $\sup_{t\geq0}\{ \overline{M}_f+f^d(t) \}$. Lastly, considering the constraints of \eqref{bound_part_force}, there are two constants, $\bunderline{\partial f}$ and $\overline{\partial f}$. Their relationship is defined as $0<\bunderline{\partial f}\leq \partial f(\chi) \leq\overline{\partial f}$.

 \subsection{Kinematic Safety-Critical Control}
Suppose that we have a desired trajectory in force $f^d(t)$, in position $\bs{p}^{d}(t)=[y^d(t), z^d(t)]^\top$ and a desired rotation $\bs{R}^d(t)$ to be tracked by the controller and the error vector for this trajectory as defined in \eqref{err_dot}. Then, if we pick the trajectory tracking controller $\dot{\bs{x}}_{c}= [\dot{{x}}_e^{{c}},\dot{\bs{p}}_e^{{c}^\top}, \bs{\omega}_e^{{c}^\top}]^\top$, in which $\dot{\bs{p}}_e^c=[\dot{y}_e^c,\dot{z}_e^c]^\top$, such that $\dot{\bs{e}}=-\gamma\bs{e}$ for a positive constant $\gamma$, a stable linear system will be resulted and the reference trajectory is being tracked. For this aim and in view of \eqref{bound_part_force}, by picking 
\begin{subequations}\label{traj_cont}
\begin{align}
&\dot{{x}}_e^{c}=\partial f(\chi)^{-1}(\dot{{f}}^d-\gamma {e}_f),\\
   &\dot{\bs{p}}_e^{c}=\dot{\bs{p}}^d-\gamma\bs{e}_p,\\
   &\bs{\omega}_e^{c}=\bs{\omega}^d-\gamma\bs{e}_o\bs{L}^{-1},
\end{align}
\end{subequations}
the exponential stability in the error dynamics, i.e., $\bs{e}(t)\leq exp(-\gamma t)\bs{e}(0)$ is achieved.
Next we propose a state feedback control protocol that incorporates the performance constraints of \eqref{eq:11}  by employing the ZCBF notion and achieves force/position/orientation tracking of the corresponding smooth and bounded desired trajectories. 
\begin{definition}\label{def1}
\cite{ames2016control} Zeroing Control Barrier Functions: Consider the system 
\begin{align}\label{sysZCBF}
    \dot{\bs{x}} =f(\bs{x})+g(\bs{x})\bs{u}
\end{align}
, where $f(\cdot)$ and $g(\cdot)$ are locally Lipschitz functions and $\bs{u}$ is constrained in a compact set $U\subset \mathbb{R}^m$. Let $\mathcal{C}_b\subset D$ be the superlevel set of a continuously differentiable function $b: D \subset \mathbb{R}^n\to\mathbb{R}$, i.e., $\mathcal{C}_b :=\{\bs{x}\in\mathbb{R}^n: b(\bs{x})\geq 0\}$. Then, $b$ is a zeroing control barrier function (ZCBF) for the system \eqref{sysZCBF} if there exists an extended class $\mathcal{K}_{\infty}$ function $\alpha$ such that 
\begin{align}\label{CBFCOND}
    \sup_{u\in U} [ \nabla b(\bs{x})^{\top}(f(\bs{x})+g(\bs{x})\bs{u})+\alpha(b(\bs{x}))]\geq 0, \bs{x}\in D.
\end{align}
\end{definition}
Consider the barrier functions of the form:    \begin{align}\label{barrier1}
    b_{i}(e_i) := \left(e_i+\bunderline{M}_{i}\right)\left(\overbar{M}_{i}-e_i\right),~i\in\{f,y,z,o_1,o_2,o_3\}.
    \end{align}
    Accordingly, the corresponding safe sets of the barrier functions \eqref{barrier1} will be as:
    \begin{align}\label{safe sets}
    \mathcal{S}_{i} := \{e_i\in\mathbb{R}: b_i(e_i)\geq 0\},~i\in\{f,y,z,o_1,o_2,o_3\}.
    \end{align}
\begin{lem}
Let the sets $\mathcal{S}_i$, $i\in\{f,y,z,o_1,o_2,o_3\}$ as the superlevel set of continuously differentiable functions $b_i:\mathbb{R}\to\mathbb{R}$ defined in \eqref{barrier1}. Then, the safety-critical velocity based controller from the quadratic problem
 \begin{align*}
        \dot {\bs{x}}^*({\bs{x}},t) = argmin_{\dot {\bs{x}}\in\mathbb{R}^6} \|-\gamma\bs{e}\|^2
    \end{align*}
\vspace{-4mm}\rm{s.t.}
\begin{align}\label{lem1}
&\bs{J}(\bs{q}) \bs{\zeta}-\dot{\bs{X}}_{d}[(\overbar{\bs{M}}-\bunderline{\bs{M}})-2(\bs{x}-\bs{x}_d)]
\nonumber\\ &+[\alpha_i(b_i)]_{i\in\{f,y,z,o_1,o_2,o_3\}}\geq \bs{0},    
    \end{align}
    in which $\overbar{\bs{M}}:=[\overbar{M}_i]\in \mathbb{R}^{6}$, $\bunderline{\bs{M}}:=[\bunderline{M}_i] \in \mathbb{R}^{6}$, ${i\in\{f,y,z,o_1,o_2,o_3\}}$, $\dot{\bs{X}}_{d}:=\dot{\bs{x}}_{d}I_6 \in \mathbb{R}^{6\times6}$, for identity matrix $I_6\in \mathbb{R}^{6\times6}$, and the functions $\alpha_i$, $i\in\{f,y,z,o_1,o_2,o_3\}$ are of class $\mathcal{K}_\infty$, ensures the forward invariance of the sets $\mathcal{S}_i$ and hence the safety satisfaction. Moreover, the closed-form solution is given by 
 \begin{align}\label{closed_form}
        \dot {\bs{x}}^*({\bs{x}},t)= \dot {\bs{x}}_{c}({\bs{x}},t)+\bs{l}(\bs{x},t)
    \end{align}
    in which $\bs{l}=[{l}_i]_{{i\in\{1,\cdots,6\}}}\in\mathbb{R}^6$ and
\begin{align}\label{closed_form}
        {l}_i= 
        \left\{ \begin{array}{l}
h_i\;\;\;\;\;\;\;\;\;\;\;\;\; \text{if}\;\; \psi_i< 0\\
0\;\;\;\;\;\;\;\;\;\;\;\;\;\;\text{if}\;\; \psi_i\geq 0, 
      \end{array}
      \right.
    \end{align}
     where $h_i$ is the $i$th element of the vector $\bs{h}= [h_i]_{{i\in\{1,\cdots,6\}}}=\frac{\bs{H}\Psi}{\|\bs{H}\|^2}$ and $\bs{H}:=[(\overbar{\bs{M}}-\bunderline{\bs{M}})-2(\bs{x}-\bs{x}_d)]I_6  \in \mathbb{R}^{6\times 6}$. Similarly, $\psi_i$ is the $i$th element of the vector $\bs{\psi}=[\psi_i]_{{i\in\{1,\cdots,6\}}}$ and $\bs{\psi} :=  \bs{J}(\bs{q}) \bs{\zeta}-\dot{\bs{X}}_{c}[(\overbar{\bs{M}}-\bunderline{\bs{M}})-2(\bs{x}-\bs{x}_d)]
+[\alpha_i(b_i)]_{i\in\{f,y,z,o_1,o_2,o_3\}} \in \mathbb{R}^{6}$, with $\dot{\bs{X}}_{c}:=\dot{\bs{x}}_c I_6$.
   Then, the controller \eqref{closed_form} utilizes $\dot {\bs{x}}_{c}({\bs{x}},t)$ whenever the system is safe with respect to the defined barrier functions, i.e., when $\Psi\geq \bs{0}$. On the other hand, in the case that $\dot {\bs{x}}_{c}({\bs{x}},t)$ doesn't enforce the safety condition, the controller \eqref{closed_form} enforces the system  until $\dot {\bs{x}}_{c}({\bs{x}},t)$ is safe again.
\end{lem}
\begin{proof}
The closed form- expression for $\dot {\bs{x}}^*({\bs{x}},t)$ is achieved by the satisfaction of KKT optimality conditions and following Definition \ref{def1} and \cite[Theorem 2]{xu2015robustness}.
\end{proof}

Subsequently, the task-space desired motion profile $\dot{\bs{x}}^*$ can be expressed equivalently in the configuration space via:
\begin{align}
{\bs{\zeta}}^r(t)=\bs{J}(\bs{q})^{\#}\dot{\bs{x}}^*+\big(\bs{I}_{n\times n}-\bs{J}(\bs{q})^{\#}\bs{J}\big(\bs{q}\big)\big)\dot{\bs{x}}^0 \in \mathbb{R}^n\label{jointtask}
\end{align}
where $\bs{J}(\bs{q})^{\#}$ denotes the generalized pseudo-inverse \cite{citeulike:6536020} of the Jacobian $\bs{J}(\bs{q})$ and $\dot{\bs{x}}^0$ denotes secondary tasks (e.g., maintaining manipulator's joint limits, increasing manipulability) to be regulated independently since they do not contribute to the end-effector's velocity \cite{Simetti2016877} (i.e., they	belong to the null space of the Jacobian  $\bs{J}(\bs{q})$)\footnote{For more details on task priority based control and redundancy resolution for MMSs the reader is referred to \cite{Simetti2016877} and \cite{Soylu2010325}.}.
\subsection{Control Barrier Function Based Velocity Control}
Given the desired configuration space motion profile ${\bs{\zeta}}^r(t)$ in \eqref{jointtask} that satisfies different operational limitations, we proceed with the design of a CBF-based velocity controller that achieves certain predefined minimum speed of response. Similar to the kinematic safety constraints, the first step is to define the velocity error vector:
\begin{align}
\bs{e}_\zeta(t)\triangleq[{e}_{\zeta_1}(t),\ldots,{e}_{\zeta_n}(t)]^\top= {{\bs{\zeta}}}(t)- {{\bs{\zeta}}}^r(t) \in \mathbb{R}^{n}\label{eq14}
\end{align}
We aim to impose bounded response on the system velocities errors $e_{\zeta_i}(t),i=1,\ldots,n$ as well by satisfying:
\begin{align}
-{\rho}_{\zeta_i}<e_{\zeta_i}(t)<{\rho}_{\zeta_i},~\forall t\geq 0~i=1,\ldots,n\label{eq:111}
\end{align}
where, ${\rho}_{\zeta_i}, i=1,\ldots,n$ is the predefined velocity bound that is set according to each system DoF.  We define the new error vector
\begin{align}
    \xi = \xi(\bs{\zeta},t) :=\rho_{{\zeta}}^{-1}e_{{\zeta}}(t)
\end{align}
where $\rho_{{\zeta}}:=\diag(\rho_{\zeta_i})$, ${i=1,\ldots,n}$. The control objective is then equivalent to maintaining the normalized error $\xi(t)$ in the set $(-1,1)$. For this aim, we define the continuously differentiable barrier function $b: \mathbb{R}^n\times\mathbb{R}_{\geq 0}\to\mathbb{R}$ and its 0-superlevel set as
\begin{align}\label{vel_barrier}
    &b({\bs{\zeta}},t) := \frac{1}{2}(1-\|\xi\|^2) = \frac{1}{2}(1-\|\xi(\bs{\zeta},t)\|^2),
    \end{align}
    \begin{align}\label{vel_set}
    &\mathcal{C}_b(t) :=\{\bs{\zeta}\in\mathbb{R}^n: \frac{1}{2}(1-d^2)\geq b(\bs{\zeta},t)\geq 0\}.
\end{align}
where the constant $0<d<1$ with $d\leq \|\xi\|^2$ is considered for the sake of controllability maintenance. The goal is to render the set $\mathcal{C}_b(t)$ forward invariant, i.e., to guarantee that $\bs{\zeta}(t)\in\mathcal{C}_b(t), \forall{t\geq 0}$, provided that $\bs{\zeta}(0)\in\mathcal{C}_b(0)$\cite{ames2016control}. By evaluating the derivative of $b(\bs{\zeta},t)$ along the system dynamics \eqref{eq4}, and by considering Definition \ref{def1} it can be concluded that
if there exists an extended class $\mathcal{K}_{\infty}$ function $\alpha$ such that the set $\mathcal{K}_v(\bs{q}, \bs{\zeta}, t)$ is non-empty for all $\bs{\zeta}$, the function $b(\bs{\zeta},t)$ is a ZCBF.
Therefore, considering the unknown  $\delta(\bs{q},\bs{\zeta},t)$ and its upper bound  $\bar\delta\geq\bs\delta(\bs{q},\bs{\zeta},t)>0, \forall(\bs{q},\bs{\zeta},t)$, we can define a conservative set $\bar {\mathcal{K}_v}(\bs{q}, \bs{\zeta}, t)$ as below
\begin{align}\label{SetKnew}
    &\bar {\mathcal{K}}_v(\bs{q}, \bs{\zeta}, t)=\{\bs{\tau}\in \mathbb{R}^n: -\xi^{\top}\rho^{-1}(\bs{\zeta} -\bs{\zeta}^r)\bs{M}^{-1}(\bs{q})\nonumber\\&\times(-\!\bs{C}({\bs{q}},\bs{\zeta}){\bs{\zeta}}\!-\!\bs{D}({\bs{q}},\bs{\zeta}){\bs{\zeta}}\!-\!\bs{g}(
\bs{q})\!+\bs{\tau})\nonumber\\&+\alpha(\frac{1}{2}(1-\|\xi\|^2))-\bar\delta\|\xi^{\top}\rho^{-1}(\bs{\zeta} -\bs{\zeta}^r)\bs{M}^{-1}(\bs{q})\|\geq 0\}.
\end{align}
The non-emptiness of $\bar {\mathcal{K}}_v(\bs{q}, \bs{\zeta}, t)$ implicates the
standard ZCBF-based condition, as in \eqref{CBFCOND}.
The control design consists of computing a controller that satisfies $\bs{\tau}(\bs{q}, \bs{\zeta}, t)\in\bar {\mathcal{K}_v}(\bs{q}, \bs{\zeta}, t)$ for all $( \bs{q}, \bs{\zeta}, t)$, given a ZCBF $b(\bs{\zeta},t)$.
\newline
\subsection{Kinematic-Consistent Dynamic Safety-Critical Control}
Now it is required to ensure that the kinematics barrier functions \eqref{barrier1} are consistent with the system dynamics. In order to achieve this, we follow an energy-based approach provided in \cite{singletary2021safety}. We will guarantee the safety of the combined dynamics of the mobile manipulator systems. First, consider the FBMMS's dynamics \eqref{eq4}. In order to bridge the kinematics to dynamics, we extend the kinematic safe set to a dynamic one. In this regard, we note that the inertia matrix $\bs{M}(\bs{q})$ is a symmetric positive-definite matrix, i.e., $\bs{M}(\bs{q})=\bs{M}(\bs{q})^T>0$. Then we have
\begin{align}\label{symm}
    \lambda_{min}(\bs{M}(\bs{q}))\|\bs{q}\|^2\leq \bs{q}^T(\bs{M}(\bs{q}))\bs{q} \leq \lambda_{max}(\bs{M}(\bs{q}))\|\bs{q}\|^2,
\end{align}
where $\lambda_{min}$ and $\lambda_{max}$ are the minimum and maximum eigenvalues (dependent on $\bs{q}$) of the inertia matrix $\bs{M}(\bs{q})$, and are positive due to the positive-definiteness property of $\bs{M}(\bs{q})$.

In addition, let 
\begin{align}\label{under approximation}
   {b_k}:=&-\frac{1}{\eta}\ln(\sum_{i} \exp(-\eta b_{i}(e_i)), 
\end{align}
with $\eta>0$ be a smooth approximation of the min-operator to encode the conjunction of barrier functions $b_{i}(e_i)$, $~i\in\{f,y,z,o_1,o_2,o_3\}$ in \eqref{barrier1}. Note that the accuracy of this approximation is proportionally related to $\eta$, and regardless of the choice of $\eta$ we have 
\begin{align}
    &-\frac{1}{\eta}\ln(\sum_{i=1}^{n} \exp(-\eta b_{i}(e_i))\leq \min (b_{i}(e_i)_{~i\in\{f,y,z,o_1,o_2,o_3\}}.
\end{align}
Therefore, $b_k\geq 0$ implies $b_{i}(e_i)\geq 0$ for all $~i\in\{f,y,z,o_1,o_2,o_3\}$, that means practically the satisfaction of the kinematic safety constraints of \eqref{barrier1}. 
\begin{definition}\label{dynkin barriers}
    Consider the kinematic barrier function {$b_k$} , the associated \emph{energy-based barrier function} is defined as 
\begin{align}\label{dynamic_barrier}
       b_d(\bs{q}, \bs{\zeta}) :=-\frac{1}{2}\bs{\zeta}^\top \bs{M}(\bs{q})\bs{\zeta}+\gamma b_k\geq 0,
    \end{align}
    with $\gamma>0$. The corresponding \emph{dynamic safety set} is defined as
    \begin{align}\label{dyn_safeSet}
    \mathcal{C}_d :=\{(\bs{q}, \bs{\zeta})\in Q\times\mathbb{R}^{n} : b_d(\bs{q}, \bs{\zeta})\geq 0\}.  
    \end{align}
\end{definition}
It can easily be shown that $ \mathcal{C}_d\subset \mathcal{C}$, and $Int(\mathcal{C})\subset \lim_{\gamma\to\infty}\mathcal{C}_d\subset\mathcal{C}$, where $\mathcal{C}$ determines the safe set corresponding to the barrier function $b_k$ \cite{singletary2021safety}. 
\begin{theorem}
    Consider the robotic system dynamic \eqref{eq4}. Let  $b_k: Q \subset\mathbb{R}^{n}\to\mathbb{R}$ the kinematic barrier function 
     , the energy-based barrier function $b_{d}$ as in \eqref{dynamic_barrier}, and the velocity-based barrier function $b$ as in \eqref{vel_barrier}. 
     Then, $b_{d}$ and $b$ are valid zeroing control barrier function on $\mathcal{C}_d$ and $\mathcal{C}_b$, respectively. 
    Furthermore, the control signal $\bs{\tau}^{\text{safe}}(\bs{q}, \bs{\zeta},t)$ that can solve the following optimization problem
    \begin{align*}
\min_{\bs{\tau}\in\mathbb{R}^n}\|\bs{\tau}-\bs{\tau}_{\text{des}}(\bs{q}, \bs{\zeta},t)\|^2
    \end{align*}
    \rm{s.t.}
    \begin{subequations}
    \begin{align}\label{cont_dyn}
-\bs{\zeta}^\top\bs{\tau}+\bs{g}^\top\bs{\zeta}-\bar{\delta}^\top\|\bs{\zeta}\|+\gamma\bs{J}_{b_k}\bs{\zeta}\geq -\alpha(b_{d}(\bs{q}, \bs{\zeta})),
    \end{align}
    \begin{align}\label{bar_dyn_1level}
  &\!-\!\xi^{\top}\!\!\rho^{\!-1}(\bs{\zeta}\! -\!\bs{\zeta}^r)\bs{M}^{\!-1}(\bs{q})(-\bs{C}({\bs{q}},\bs{\zeta}){\bs{\zeta}}\!-\!\bs{D}({\bs{q}},\bs{\zeta}){\bs{\zeta}}\!-\!\bs{g}(\bs{q})\!+\!\bs{\tau}\!)\nonumber\\&\!\!-\!\bar\delta\|\xi^{\top}\rho^{-1}(\bs{\zeta}\! -\!\bs{\zeta}^r)\bs{M}^{-1}(\bs{q})\|\geq -\!\alpha(b({\bs{\zeta}},t)), 
\end{align}
\end{subequations}
 in which $\bs{\tau}_{\text{des}}(\bs{q}, \bs{\zeta},t)$ is a given desired stabilizing controller, $\bar\delta\geq\bs\delta(\bs{q},\bs{\zeta},t)>0, \forall(\bs{q},\bs{\zeta},t)$ and $\bs{J}_{b_k}=\frac{\partial b_k}{\partial \bs{q}}$, guarantees the forward invariance (safety) of $\mathcal{C}_d$. 
 \end{theorem}
 \begin{proof}
     By differentiating the barrier function $b_{d}$ along the system dynamics \eqref{eq4}, utilizing the skew symmetric property of the matrix $\dot{\bs{M}(\bs{q})}-\bs{C}({\bs{q}},\bs{\zeta})-\bs{D}({\bs{q}},\bs{\zeta})$, and the upper bound $\bar\delta$ 
     , we get the inequality \eqref{cont_dyn} for $\dot b_{d}\geq \alpha(b_{d}(\bs{q}, \bs{\zeta}))$. According to Definition \ref{def1} if \eqref{cont_dyn} has a solution, then $b_{d}$ is a valid ZCBF. 
 Moreover, getting the derivative of the barrier function \eqref{vel_barrier} with respect to the system dynamics \eqref{eq4} and applying the inequality condition in \eqref{CBFCOND} results in the inequality \eqref{bar_dyn_1level}.
 Then, the feasibility of \eqref{bar_dyn_1level} for
 $(\bs{q}, \bs{\zeta}, t)$ guarantees that $b(\bs{\zeta},t)$ is a ZCBF and the controller satisfying $\bs{\tau}( \bs{q},\bs{\zeta},t)\in\bar {\mathcal{K}_v}(\bs{q},\bs{\zeta}, t)$ (in \eqref{SetKnew}) for all $(\bs{q}, \bs{\zeta}, t)$. 
 \end{proof}
\section{Simulation results}
The simulation results using MATLAB$^\circledR$ designed for an underwater vehicle manipulator system (UVMS)  with inherent uncertainties and external disturbances \cite{nikou2021robust}, equipped with a small 4 DoFs manipulator based on the Newton-Euler approach \cite{Schjolberg94modellingand}.
We consider a scenario involving 3D motion, where the end-effector is in compliant contact with a planar surface with stiffness model $f=k\chi^2$,  $k=300 \frac{N}{m^2}$.
The UVMS is in contact with the compliant environment exerting a force normal to the surface $f(0)=0.45~N$. The  desired force, normal to the surface direction, is set as $f^d(t)=1N$. The UVMS should track a desired position trajectory on the surface and attain a perpendicular orientation (i.e., $ \bs{R}^d=I_{3\times 3}$) with respect to the surface. We choose  $\bunderline{f}^*=0.2N$ and $\overline{f}^*=1.8N$ such that the contact with the surface isn't lost and that excessive interaction forces are avoided. We set the parameters $\bunderline{\bs{M}}:=[1.0,\;0.1,\;0.1,\;0.3,\; 0.2,\; 0.2 ]$,  $\overbar{\bs{M}}:=[0.5,\;0.1,\;0.1,\;0.3,\;0.2,\;0.2]$, and $\rho_{\zeta_i}=0.5$.
The dynamics of the UVMS are affected by external disturbances in the form of slowly time varying sea currents acting along $x$, $y$ and $z$ axes modeled by the corresponding velocities  ${v}_i^{\{I\}}=0.1\sin(\frac{\pi}{25}t)\frac{m}{s},~i\in\{x,y,z\}$.  Finally, bounded measurement noise of normal distribution with $5\%$ standard deviation was considered. 
The results are depicted  in Fig.\ref{fig:force}-Fig.\ref{fig:case1_joint}.  
Fig.\ref{fig:force} presents the evolution of the actual force exerted by the UVMS with respect to the desired force profile. It can be seen that the force exerted by the UVMS remained inside the desired region and the contact is never lost. The evolution of the errors at the first and second level of the proposed controller are illustrated in Fig.\ref{fig:case1_cartesian}  and Fig.\ref{fig:case1_joint}, respectively. It can be concluded that even with the influence of external disturbances as well as measurements noise, the errors in all directions remain in the predefined safe sets.  
\begin{figure}[h!]
	\centering
	\includegraphics[scale = 0.35]{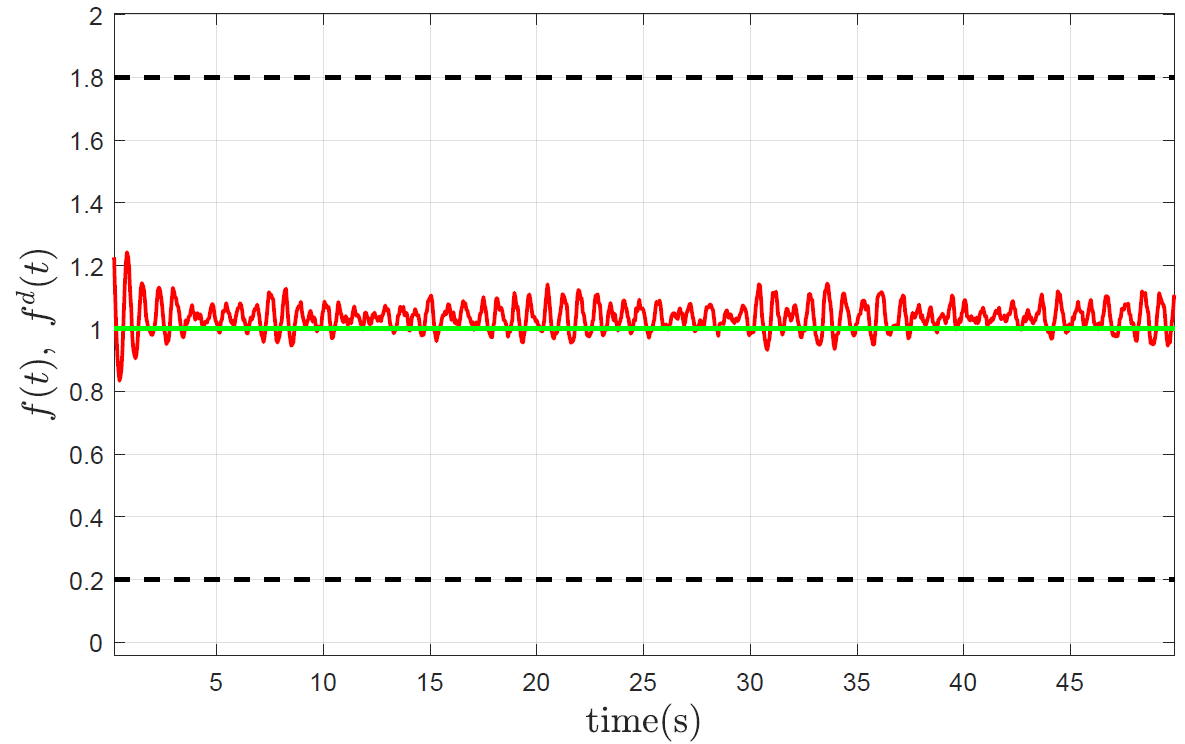}
\vspace{-3mm}\caption{The evolution of the force trajectory. The desired constant force and the actual force exerted by the UVMS are indicated by green and red color respectively. }
	\label{fig:force}       
\end{figure}
\begin{figure}[h!]
	\centering
	\includegraphics[scale = 0.8]{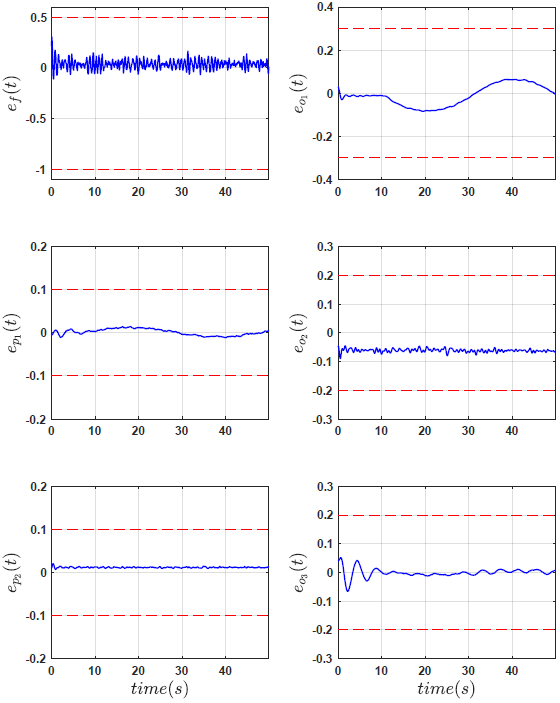}
\vspace{-3mm}\caption{ The evolution of the errors at the first level of the proposed control scheme. The errors and performance bounds are indicated by blue and red color respectively.}\vspace{-2mm}
	\label{fig:case1_cartesian}       
\end{figure}
\begin{figure}[h!]
	\centering
\includegraphics[width=3.4in]{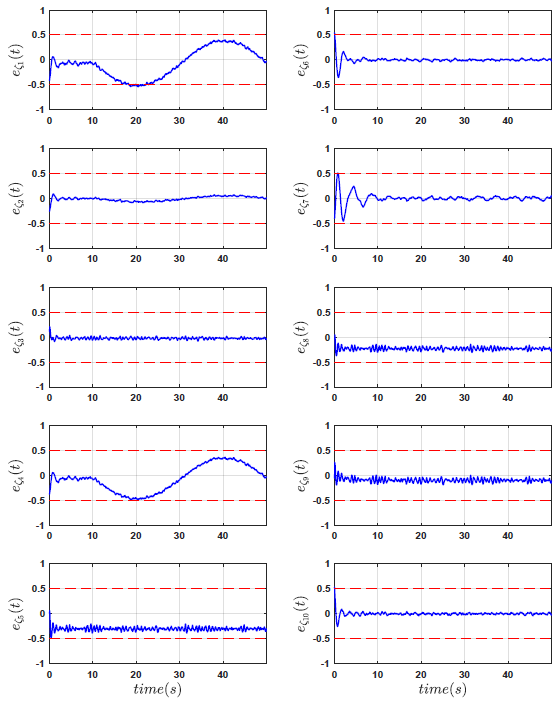}
\vspace{-3mm}	\caption{The evolution of the errors at the second level of the proposed control scheme. The errors and safety bounds are indicated by blue and red color respectively.}\vspace{-3mm}
	\label{fig:case1_joint}       
\end{figure}

\section{Conclusions}\label{Conclusions}
This paper introduces a robust force/position control strategy for free floating Mobile Manipulator Systems that interact compliantly with their environment, finding notable applications in areas such as welding. The advanced control design adapts to uncertainties in system dynamic parameters and the stiffness model. It ensures a predetermined behavior by controlling desired overshoot and exhibits resilience against external disturbances and uncertainties. By integrating the Zeroing Control Barrier Function the proposed method accounts for operational challenges, including joint constraints, kinematic singularities, contact maintenance, and performance boundaries related to trajectory and rotation tracking. As a result, the proposed controller not only meets but efficiently manages all these constraints and limitations.



\bibliographystyle{IEEEtran}
\bibliography{Ref.bib}

\begin{thebibliography}{10}
\providecommand{\url}[1]{#1}
\csname url@samestyle\endcsname
\providecommand{\newblock}{\relax}
\providecommand{\bibinfo}[2]{#2}
\providecommand{\BIBentrySTDinterwordspacing}{\spaceskip=0pt\relax}
\providecommand{\BIBentryALTinterwordstretchfactor}{4}
\providecommand{\BIBentryALTinterwordspacing}{\spaceskip=\fontdimen2\font plus
\BIBentryALTinterwordstretchfactor\fontdimen3\font minus
  \fontdimen4\font\relax}
\providecommand{\BIBforeignlanguage}[2]{{%
\expandafter\ifx\csname l@#1\endcsname\relax
\typeout{** WARNING: IEEEtran.bst: No hyphenation pattern has been}%
\typeout{** loaded for the language `#1'. Using the pattern for}%
\typeout{** the default language instead.}%
\else
\language=\csname l@#1\endcsname
\fi
#2}}
\providecommand{\BIBdecl}{\relax}
\BIBdecl

\bibitem{nikou2017decentralized}
A.~Nikou, S.~Heshmati-Alamdari, C.~K. Verginis, and D.~V. Dimarogonas,
  ``Decentralized abstractions and timed constrained planning of a general
  class of coupled multi-agent systems,'' in \emph{2017 IEEE 56th Annual
  Conference on Decision and Control (CDC)}.\hskip 1em plus 0.5em minus
  0.4em\relax IEEE, 2017, pp. 990--995.

\bibitem{sereinig2020review}
M.~Sereinig, W.~Werth, and L.-M. Faller, ``A review of the challenges in mobile
  manipulation: systems design and robocup challenges,'' \emph{e \& i
  Elektrotechnik und Informationstechnik}, vol. 137, no.~6, pp. 297--308, 2020.

\bibitem{heshmati2023control}
S.~Heshmati-Alamdari, G.~C. Karras, M.~Sharifi, and G.~K. Fourlas, ``Control
  barrier function based visual servoing for underwater vehicle manipulator
  systems under operational constraints,'' in \emph{2023 31st Mediterranean
  Conference on Control and Automation (MED)}.\hskip 1em plus 0.5em minus
  0.4em\relax IEEE, 2023, pp. 710--715.

\bibitem{baizid2017behavioral}
K.~Baizid, G.~Giglio, F.~Pierri, M.~A. Trujillo, G.~Antonelli, F.~Caccavale,
  A.~Viguria, S.~Chiaverini, and A.~Ollero, ``Behavioral control of unmanned
  aerial vehicle manipulator systems,'' \emph{Autonomous Robots}, vol.~41,
  no.~5, pp. 1203--1220, 2017.

\bibitem{karras2022image}
G.~C. Karras, G.~K. Fourlas, A.~Nikou, C.~P. Bechlioulis, and
  S.~Heshmati-Alamdari, ``Image based visual servoing for floating base mobile
  manipulator systems with prescribed performance under operational
  constraints,'' \emph{Machines}, vol.~10, no.~7, p. 547, 2022.

\bibitem{heshmati2019distributed}
S.~Heshmati-Alamdari, G.~C. Karras, and K.~J. Kyriakopoulos, ``A distributed
  predictive control approach for cooperative manipulation of multiple
  underwater vehicle manipulator systems,'' in \emph{2019 international
  conference on robotics and automation (ICRA)}.\hskip 1em plus 0.5em minus
  0.4em\relax IEEE, 2019, pp. 4626--4632.

\bibitem{heshmati2017robust}
S.~Heshmati-alamdari, A.~Nikou, K.~J. Kyriakopoulos, and D.~V. Dimarogonas, ``A
  robust force control approach for underwater vehicle manipulator systems,''
  \emph{IFAC-PapersOnLine}, vol.~50, no.~1, pp. 11\,197--11\,202, 2017.

\bibitem{rani2018efficient}
M.~Rani, N.~Kumar, and H.~P. Singh, ``Efficient position/force control of
  constrained mobile manipulators,'' \emph{International Journal of Dynamics
  and Control}, vol.~6, pp. 1629--1638, 2018.

\bibitem{dong2002trajectory}
W.~Dong, ``On trajectory and force tracking control of constrained mobile
  manipulators with parameter uncertainty,'' \emph{Automatica}, vol.~38, no.~9,
  pp. 1475--1484, 2002.

\bibitem{wang2017safety}
L.~Wang, A.~D. Ames, and M.~Egerstedt, ``Safety barrier certificates for
  collisions-free multirobot systems,'' \emph{IEEE Transactions on Robotics},
  vol.~33, no.~3, pp. 661--674, 2017.

\bibitem{sharifi2022higher}
M.~Sharifi and D.~V. Dimarogonas, ``Higher order barrier certificates for
  leader-follower multi-agent systems,'' \emph{IEEE Transactions on Control of
  Network Systems}, 2022.

\bibitem{alan2023control}
A.~Alan, A.~J. Taylor, C.~R. He, A.~D. Ames, and G.~Orosz, ``Control barrier
  functions and input-to-state safety with application to automated vehicles,''
  \emph{IEEE Transactions on Control Systems Technology}, 2023.

\bibitem{kolathaya2023energy}
S.~Kolathaya, ``Energy based control barrier functions for robotic systems,''
  \emph{Authorea Preprints}, 2023.

\bibitem{antonelli}
G.~Antonelli, \emph{``{U}nderwater {R}obots"}, ser. Springer Tracts in Advanced
  Robotics.\hskip 1em plus 0.5em minus 0.4em\relax Springer International
  Publishing, 2013.

\bibitem{johnson1987contact}
K.~L. Johnson and K.~L. Johnson, \emph{Contact mechanics}.\hskip 1em plus 0.5em
  minus 0.4em\relax Cambridge university press, 1987.

\bibitem{arimoto2000dynamics}
S.~Arimoto, P.~T.~A. Nguyen, H.-Y. Han, and Z.~Doulgeri, ``Dynamics and control
  of a set of dual fingers with soft tips,'' \emph{Robotica}, vol.~18, no.~1,
  pp. 71--80, 2000.

\bibitem{caccavale1998resolved}
F.~Caccavale, C.~Natale, B.~Siciliano, and L.~Villani, ``Resolved-acceleration
  control of robot manipulators: A critical review with experiments,''
  \emph{Robotica}, vol.~16, no.~5, pp. 565--573, 1998.

\bibitem{sciavicco2012modelling}
L.~Sciavicco and B.~Siciliano, \emph{Modelling and control of robot
  manipulators}.\hskip 1em plus 0.5em minus 0.4em\relax Springer Science \&
  Business Media, 2012.

\bibitem{ames2016control}
A.~D. Ames, X.~Xu, J.~W. Grizzle, and P.~Tabuada, ``Control barrier function
  based quadratic programs for safety critical systems,'' \emph{IEEE
  Transactions on Automatic Control}, vol.~62, no.~8, pp. 3861--3876, 2016.

\bibitem{xu2015robustness}
X.~Xu, P.~Tabuada, J.~W. Grizzle, and A.~D. Ames, ``Robustness of control
  barrier functions for safety critical control,'' \emph{IFAC-PapersOnLine},
  vol.~48, no.~27, pp. 54--61, 2015.

\bibitem{citeulike:6536020}
B.~Siciliano and J.~J.~E. Slotine, ``{A general framework for managing multiple
  tasks in highly redundant robotic systems},'' \emph{Advanced Robotics, 1991.
  'Robots in Unstructured Environments', 91 ICAR., Fifth International
  Conference on}, pp. 1211--1216 vol.2, 1991.

\bibitem{Simetti2016877}
E.~Simetti and G.~Casalino, ``A novel practical technique to integrate
  inequality control objectives and task transitions in priority based
  control,'' \emph{Journal of Intelligent and Robotic Systems: Theory and
  Applications}, vol.~84, no. 1-4, pp. 877--902, 2016.

\bibitem{Soylu2010325}
S.~Soylu, B.~Buckham, and R.~Podhorodeski, ``Redundancy resolution for
  underwater mobile manipulators,'' \emph{Ocean Engineering}, vol.~37, no. 2-3,
  pp. 325--343, 2010.

\bibitem{singletary2021safety}
A.~Singletary, S.~Kolathaya, and A.~D. Ames, ``Safety-critical kinematic
  control of robotic systems,'' \emph{IEEE Control Systems Letters}, vol.~6,
  pp. 139--144, 2021.

\bibitem{nikou2021robust}
A.~Nikou, C.~K. Verginis, S.~Heshmati-alamdari, and D.~V. Dimarogonas, ``A
  robust non-linear mpc framework for control of underwater vehicle manipulator
  systems under high-level tasks,'' \emph{IET Control Theory \& Applications},
  vol.~15, no.~3, pp. 323--337, 2021.

\bibitem{Schjolberg94modellingand}
I.~Schj{\o}lberg and T.~I. Fossen, ``Modelling and control of underwater
  vehicle-manipulator systems,'' in \emph{in Proc. rd Conf. on Marine Craft
  maneuvering and control}.\hskip 1em plus 0.5em minus 0.4em\relax Citeseer,
  1994.

\end{thebibliography}

\end{document}